\documentclass{article}

\usepackage[preprint]{style}

\usepackage[utf8]{inputenc}
\usepackage[T1]{fontenc}
\usepackage{hyperref}
\usepackage{caption}
\usepackage{booktabs}
\usepackage{amsmath}
\usepackage{amssymb}
\usepackage{amsthm}
\usepackage{microtype}
\usepackage{graphicx}
\usepackage{wrapfig}
\usepackage{enumitem}

\microtypesetup{expansion=false}
\hypersetup{hypertexnames=false}
\graphicspath{{figures/}}
\title{Distributional Reward Calibration and Budget-Neutral Credit Assignment for Multimodal RLVR}

\author{Lang Chen, Zongxin Xu, Borui Wang, Anran Liu, and Shitong Chen\thanks{Corresponding author, Shenzhen University.}}

\begin{document}

\maketitle

\begin{abstract}
We revisit advantage estimation in Group Relative Policy Optimization (GRPO) and observe that any token-level advantage admits a two-stage factorization: a \emph{sample-level baselining} step that converts scalar verifier rewards into a comparable sequence-level signal, and a \emph{token-level credit assignment} step that distributes that signal over the trajectory. Existing estimators typically address only one of these stages: affine baselines preserve the shape of each task's reward law, so heavy tails and binary cliffs propagate into the gradient, while token-level credit rules tie the size of each token's update to raw confidence scales and need not conserve the total credit issued by the verifier. We propose \textbf{TACE} (Two-stage Advantage and Credit Estimation), a two-stage estimator that addresses both stages. At stage~(S1), we replace linear standardization with a one-dimensional optimal-transport map to a standard normal, which gives a common population reward law across tasks and caps finite-sample outliers deterministically. At stage~(S2), we rank-calibrate a stop-gradient evidence contrast within each trajectory and renormalize the clipped token weights to have mean one, so credit is redistributed across tokens while its total mass and sign are conserved. We prove population-level inter-task moment equity, a deterministic tail bound, and mass-conservation, boundedness, and sign-anchoring properties that affine baselines and unnormalized token weights do not satisfy. On a multimodal reasoning suite, TACE obtains the strongest results on a majority of the benchmarks but does not dominate every setting.
\end{abstract}

\section{Introduction}
\label{sec:intro}

GRPO reduces reinforcement learning with verifiable rewards to a simple estimator: it standardizes scalar rewards into a sequence-level advantage and broadcasts that advantage uniformly over tokens. This practice conflates two independent decisions: \textbf{(S1) sample-level baselining}, which turns a scalar reward into a comparable advantage, and \textbf{(S2) token-level credit assignment}, which distributes that advantage over the trajectory's tokens. Existing variants typically address \emph{exactly one} stage: Dr.\,GRPO~\citep{liu2025understandingr1zeroliketrainingcriticaldrgrpo} and EMA-GRPO~\citep{feng2025onethinkerallinonereasoningmodel} refine the affine form of the baseline while leaving token credit uniform, whereas recent token-crediting rules introduce per-token signals while retaining an affine sequence-level baseline~\citep{yang2026selfdistilledrlvr}.

The two stages carry complementary failure modes that neither alone can repair. An affine baseline cannot reshape a reward law: heavy tails, binary cliffs, and inter-task variance gaps survive the rescaling and propagate verbatim into the gradient magnitude and then into every token of the trajectory. Conversely, magnitude-based token weights address allocation but tie the size of each token's credit to the raw scale of the underlying signal, so token reallocation can silently amplify or attenuate the total signal the verifier issued. What is missing is a two-stage estimator in which the sample baseline is \emph{distribution-aware} rather than affine, and the token channel is \emph{budget-neutral}: it may redistribute credit, never inflate it.

We propose \textbf{TACE}, a two-stage advantage estimator that gives each stage the invariant it needs: distribution-aware rank calibration at stage~(S1) and budget-neutral credit reallocation at stage~(S2). Stage~(S1) uses the population 1D optimal-transport (OT) map onto $\mathcal N(0,1)$, implemented by $A_i = \Phi^{-1}(\widehat F_\tau(R_i))$; it calibrates the advantage by the \emph{rank} of the reward rather than its raw value, yielding population-level inter-task equity, a deterministic finite-sample tail cap, and order preservation at one $O(N\log N)$ sort per group. Stage~(S2) builds the token channel on a stop-gradient evidence contrast, as used in previous work~\citep{yang2026selfdistilledrlvr,zhao2026opsd,hbotter2026reinforcementlearningselfdistillation_sdpo}, and introduces two modifications: the per-token evidence difference is rank-calibrated \emph{within the trajectory}, removing its dependence on teacher confidence and model calibration; and the clipped, sign-anchored weights are renormalized to mean one, so the total credit mass issued by the verifier is exactly conserved while its allocation across tokens is sharpened. Because the two rank maps act on disjoint indices (samples vs.\ tokens) and both preserve signs, the stages compose without interference; the net cost over vanilla GRPO is one sort per group, one sort per trajectory, and one additional privileged-context pass.

In summary, our contributions are three-fold:
\begin{itemize}[nosep,leftmargin=*]
    \item We propose a distribution-aware stage-(S1) baseline, replacing GRPO's affine standardization with a 1D OT map onto $\mathcal N(0,1)$, and prove population-level inter-task moment equity (not just two-moment matching) plus a deterministic finite-sample outlier cap.
    \item We introduce a rank-calibrated, mass-preserving stage-(S2) credit rule: starting from the direction-anchored evidence ratio in previous work, we rank-calibrate the token evidence within each trajectory and renormalize the clipped weights to mean one. We prove positivity, a two-sided deterministic bound, exact credit-mass conservation, sign anchoring, and invariance to strictly increasing recalibration of the evidence signal; because the two stages act on disjoint indices and preserve signs, the composed estimator also enjoys a composition guarantee.
    \item We provide evidence on multimodal RLVR: on a multimodal reasoning suite, the composed estimator obtains the strongest results on a majority of the benchmarks while exposing the settings where further calibration is needed.
\end{itemize}

\section{Related Work}

\subsection{Reward Normalization Inside GRPO}
Since GRPO's introduction~\citep{grpo,deepseekr1}, a substantial body of work has revisited how the group-relative advantage $A_i = (R_i-\mu_G)/\sigma_G$ should be standardized when the verifier signal is heterogeneous across tasks or rollouts. Dr.\,GRPO~\citep{liu2025understandingr1zeroliketrainingcriticaldrgrpo} drops the per-group standard deviation and replaces it with a centered baseline; EMA-GRPO~\citep{feng2025onethinkerallinonereasoningmodel} swaps the local $\sigma_G$ for a task-wise running estimate to mitigate variance imbalance across multimodal task families; further extensions explore softer or task-conditioned variants~\citep{yu2025dapo,gao2025softadaptivepolicyoptimization,xie2025logic,zheng2025group}. All of these schemes apply an \emph{affine} transformation to the raw reward, which by construction cannot reshape the underlying reward distribution: heavy tails, binary cliffs, and inter-task scale gaps survive the rescaling and propagate into the gradient magnitude. TACE departs from this affine family by applying a non-linear, rank-based transport to a fixed reference distribution, recovering inter-task equity and bounded tails as automatic consequences rather than as side products of hyperparameter tuning.

\subsection{Advantage Reallocation in Verifier-Only RL}
Fine-grained supervision for reasoning can be organized along two axes: where a token score comes from, and what invariant the score is required to obey. Process supervision learns an intermediate judge from annotations, generated rationales, or model-based checks~\citep{lightman2024lets,wang2024math,luo2024improve,chen2024step,zhang2024generative,cui2025process,dai2025sgrpoearlyexitreinforcement,yang2025dynamicearlyexitreasoning,yang2025testtimepromptintervention}. Verifier-only alternatives derive a relative token score from rollout statistics such as entropy, salience, attention, or outcome sensitivity~\citep{cheng2026reasoning,wangbeyond,chen2025seed,sun2025ktaemodelfreealgorithmkeytokens,xie2025unlocking,li2025attention,chen2025beyond,li2026outcome}. The former introduces a separate supervision surface; the latter usually controls token emphasis without a guarantee on the total credit issued to a trajectory. TACE addresses the second invariant directly: its within-trajectory rank map supplies a common evidence scale, while mean-one renormalization turns token learning into reallocation rather than amplification.

\subsection{Privileged Context as a Magnitude Signal}
Distillation-style training offers a useful diagnostic for this design. External on-policy teachers can provide position-wise targets, but they require a second inference stack and compatible output spaces~\citep{agarwal2024gkd,lu2025onpolicy,fu2026revisitingonpolicydistillationempirical,mimo2025flash}. Self-conditioning removes the extra model but creates an observability mismatch when the conditioned copy sees the answer~\citep{zhao2026opsd,hbotter2026reinforcementlearningselfdistillation_sdpo}. The analysis of \citet{yang2026selfdistilledrlvr} therefore confines privileged evidence to a stop-gradient multiplier whose sign is governed by the verifier. TACE adopts that confinement, but changes the multiplier's calibration: evidence values are converted to within-trajectory ranks, clipped, and renormalized to mean one. Consequently, a monotone change in teacher confidence cannot alter the allocation, and the trajectory-level verifier decision remains the only source of aggregate reinforcement.

\subsection{Optimal Transport for Language Model Training}
Optimal transport (OT) has appeared in language-model post-training mainly as a distance functional: token-level Wasserstein distances regularize preference alignment, stochastic-dominance constraints reshape reward distributions, and OT-based projections enforce safety properties in latent space~\citep{melnyk2024distributionalpreferencealignmentllms,li2025optimaltransportbasedtokenweighting,na2026semanticawarewassersteinpolicyregularization,nanfack2026efficientrefusalablationllm}. The role we give OT is different: rather than measuring a divergence between two learned distributions, we exploit the closed-form 1D Brenier map onto $\mathcal{N}(0,1)$ as the canonical solution to the sample-level baselining problem inside GRPO. This places OT inside the advantage estimator itself, where it provides distribution invariance, outlier robustness, and order preservation as immediate properties of the rank-CDF composition.

\subsection{Multimodal RL Post-Training}
Reinforcement learning has rapidly become a standard tool for sharpening multimodal reasoning in vision-language models~\citep{li2025star,sun2025reinforcement,feng2025video,sun2025spacevista,zhou2025reinforced,duan2025codeplot,chen2025advancing,chen2025ares,meng2025open,wang2025perceptionawarepolicyoptimizationmultimodal,zhang2026lessrightbidirectionalperceptual}, where reward heterogeneity across perception-heavy and reasoning-heavy tasks is far more pronounced than in pure-text settings. Existing recipes typically address this with task-specific data curation, auxiliary perceptual objectives, or external captioners~\citep{tian2025more,yang2025learninglookdisentangledcurriculum,tu2025perceptionconsistencymultimodallargelanguage}. TACE is orthogonal to these recipes: by enforcing distribution invariance at stage~(S1), it removes reward-calibration differences from the advantage scale, so the same estimator can ingest perception, mathematics, and chart rewards without per-task scale tuning.

\section{Background and Diagnosis}
\label{sec:background}

\newtheorem{theorem}{Theorem}
\newtheorem{proposition}{Proposition}

Section~\ref{sec:prelim} fixes the notation we use for GRPO and on-policy distillation.

\subsection{Preliminaries}
\label{sec:prelim}

\paragraph{GRPO.}
RLVR post-training assumes a policy $\pi_\theta$ generating responses $y=(y_1,\dots,y_T)$ autoregressively from query $x$ under a scalar verifier $R(x,y)$, often binary. GRPO~\citep{grpo} samples $G$ candidates $\{y^{(i)}\}\sim\pi_{\theta_{\mathrm{old}}}(\cdot\mid x)$ and standardizes their rewards into a sequence-level advantage $A^{(i)} = (R(x,y^{(i)}) - \mu_G)/\sigma_G$, then reuses the same $A^{(i)}$ for every token of $y^{(i)}$ in the clipped policy-gradient objective
\begin{equation}
\mathcal{L}_{\mathrm{GRPO}}(\theta) = \mathbb{E}\!\left[\frac{1}{G}\sum_{i=1}^{G}\frac{1}{|y^{(i)}|}\sum_{t=1}^{|y^{(i)}|} \min\!\left(\rho_t^{(i)} A^{(i)},\; \mathrm{clip}(\rho_t^{(i)},1\!-\!\epsilon,1\!+\!\epsilon)\,A^{(i)}\right)\right],
\label{eq:grpo}
\end{equation}
where $\rho_t^{(i)}$ is the standard importance ratio. Per-token credit within a trajectory is necessarily uniform.

\paragraph{On-policy and self-distillation objectives.}
A second line densifies token-level supervision through a reference $P_T$. On-policy distillation (OPD)~\citep{agarwal2024gkd} takes $P_T$ from a separate stronger model on the same query $x$. On-policy self-distillation (OPSD)~\citep{zhao2026opsd,hbotter2026reinforcementlearningselfdistillation_sdpo} reuses $\pi_\theta$ as both player and reference, with the reference receiving privileged input $r$ (in our setting, the final ground-truth answer). With $P_S(\cdot\mid y_{<t})\triangleq \pi_\theta(\cdot\mid x, y_{<t})$ and $P_T(\cdot\mid y_{<t})\triangleq \pi_\theta(\cdot\mid x, r, y_{<t})$, both fit
\begin{equation}
\mathcal{L}_{\mathrm{OP(S)D}}(\theta) = \mathbb{E}_{(x,r)\sim\mathcal{S}}\,\mathbb{E}_{y\sim P_S(\cdot\mid x)}\!\left[\frac{1}{|y|}\sum_{t=1}^{|y|} D\!\bigl(P_T \,\Vert\, P_S\bigr)\right].
\label{eq:opsd}
\end{equation}
Here $D$ denotes a chosen statistical divergence, such as KL or JS. The shared design makes $P_T$ determine the gradient \emph{direction}, a choice TACE cannot inherit because it violates direction anchoring.

\section{Method}
\label{sec:method}

\begin{figure}[t]
    \centering
    \includegraphics[width=0.92\linewidth]{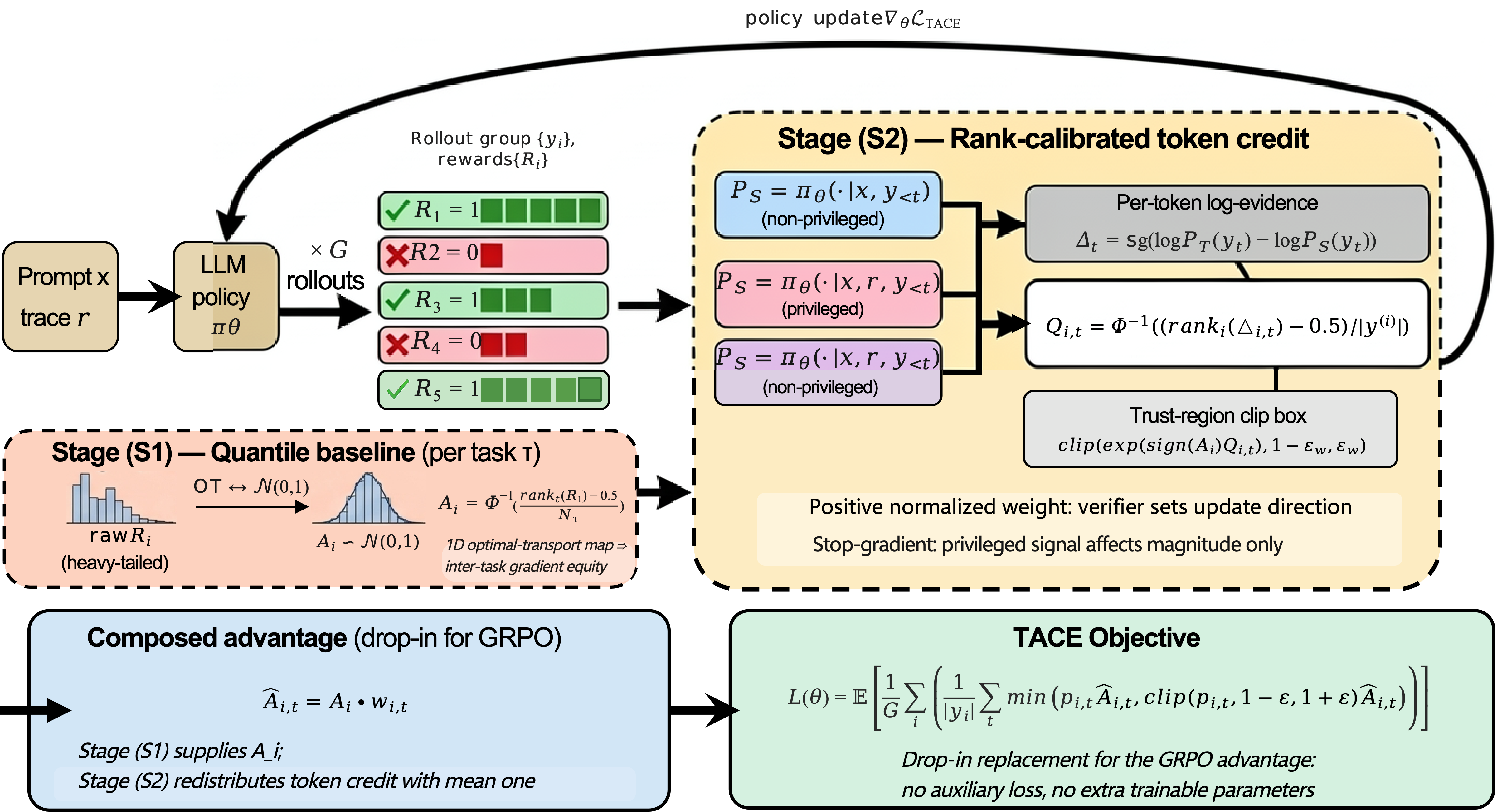}
    \vspace{-0.6em}
    \caption{\textbf{Overview of TACE.} \textbf{Stage~(S1), ours:} replaces $(R_i-\mu_G)/\sigma_G$ with the 1D optimal-transport map onto $\mathcal{N}(0,1)$, $A_i = \Phi^{-1}((\mathrm{rank}_\tau(R_i)-0.5)/N_\tau)$, for inter-task gradient equity. \textbf{Stage~(S2), ours (starting from the evidence ratio of \citealp{yang2026selfdistilledrlvr}):} runs $\pi_\theta$ under non-privileged ($P_S$) and privileged ($P_T$) contexts, rank-calibrates the per-token log-evidence within the trajectory, and applies a sign-anchored, mean-one (mass-preserving) magnitude weight $w_{i,t}$. The composed token advantage $\hat{A}_{i,t} = A_i\cdot w_{i,t}$ plugs into the standard GRPO clipped surrogate.}
    \label{fig:framework}
    \vspace{-0.5em}
\end{figure}

Figure~\ref{fig:framework} gives an overview of the composed estimator. One architectural requirement governs stage~(S2): privileged evidence may influence how strongly a token is credited, but the verifier remains the sole source of reinforcement or suppression. \textbf{TACE} (\emph{Two-stage Advantage and Credit Estimation}) preserves this separation while giving each stage its corresponding invariant: stage~(S1) replaces linear standardization with a distribution-aware OT baseline, and stage~(S2) rank-calibrates the evidence signal and conserves credit mass.

\subsection{A Two-Stage View of GRPO Advantage Estimation}
\label{sec:design-principle}

GRPO writes its advantage as $A_i = (R_i-\mu_G)/\sigma_G$ with the same $A_i$ shared by every token in $y^{(i)}$, but this hides two statistically independent operations:
\begin{equation}
    \underbrace{\{R_i\}_{i=1}^{G} \;\xrightarrow{\;\text{(S1)}\;}\; A_i}_{\text{sample-level baselining}}
    \;\xrightarrow{\;\text{(S2)}\;}\;
    \underbrace{\hat{A}_{i,t} \;=\; A_i \cdot w_{i,t}}_{\text{token-level credit assignment}}.
    \label{eq:two-stage}
\end{equation}
Vanilla GRPO uses linear standardization in (S1) and the trivial weight $w_{i,t}=1$ in (S2). The design space in Table~\ref{tab:unified-method} is therefore almost diagonal: Dr.\,GRPO~\citep{liu2025understandingr1zeroliketrainingcriticaldrgrpo} and EMA-GRPO~\citep{feng2025onethinkerallinonereasoningmodel} alter only the affine form of $A_i$, whereas RLSD~\citep{yang2026selfdistilledrlvr} modifies only the token multiplier. Each stage carries a distinct failure mode the other cannot repair: an affine baseline cannot reshape a reward law, so distributional differences remain in the update magnitude; uniform credit cannot distinguish pivotal tokens from filler, while using $P_T/P_S$ as a gradient target is structurally leakage-prone~\citep{yang2026selfdistilledrlvr}.

\paragraph{Two invariants any principled estimator must satisfy.}
\textbf{I1 (Distribution invariance):} the marginal distribution of $A_i$ should not depend on the verifier's calibration; (S1) must be a \emph{distribution-aware} map, not an affine one. \textbf{I2 (Direction anchoring):} the sign of every token-level advantage must equal $\mathrm{sign}(A_i)$, the verifier's verdict on the trajectory; (S2) must act on the \emph{magnitude} of credit, never on its direction. Section~\ref{sec:s1-quantile} instantiates I1 via the 1D optimal-transport map onto $\mathcal{N}(0,1)$; Section~\ref{sec:s2-rlsd} instantiates I2 via a stop-gradient evidence ratio whose exponent is $\mathrm{sign}(A_i)$.

\begin{table}[t]
\caption{Representative estimators viewed through Eq.~\eqref{eq:two-stage}. TACE contributes the distribution-aware stage-(S1) baseline and composes it with a rank-calibrated, mass-preserving version of direction-anchored credit.}
\label{tab:unified-method}
\centering
\footnotesize
\setlength{\tabcolsep}{0.4em}
\resizebox{\linewidth}{!}{%
\begin{tabular}{llll}
\toprule
\textbf{Method} & \textbf{Sample baseline $A_i$ (S1)} & \textbf{Token weight $w_{i,t}$ (S2)} & \textbf{Direction anchor} \\
\midrule
GRPO~\citep{grpo}              & $(R_i-\mu_G)/\sigma_G$                       & $1$                                              & Verifier \\
Dr.\,GRPO~\citep{liu2025understandingr1zeroliketrainingcriticaldrgrpo}  & $R_i-\mu_G$ & $1$                                              & Verifier \\
EMA-GRPO~\citep{feng2025onethinkerallinonereasoningmodel}    & $(R_i-\mu_G)/\sigma_\tau^{(t)}$              & $1$                                              & Verifier \\
OPSD~\citep{zhao2026opsd}      & \textemdash                                   & $P_T/P_S$ as gradient target                     & \textbf{Teacher} \\
SDPO~\citep{hbotter2026reinforcementlearningselfdistillation_sdpo} & \textemdash & sign-free evidence ratio                       & \textbf{Teacher} \\
\midrule
RLSD~\citep{yang2026selfdistilledrlvr} & $(R_i-\mu_G)/\sigma_G$ & $\mathrm{clip}\!\bigl((P_T/P_S)^{\mathrm{sign}(A_i)}\bigr)$ & Verifier \\
\midrule
\textbf{TACE (Ours)} & $\Phi^{-1}\!\bigl(\widehat{F}_\tau(R_i)\bigr)$ & mean-one $\mathrm{clip}\!\bigl(e^{\mathrm{sign}(A_i)Q_{i,t}}\bigr)$, $Q_{i,t}{=}\Phi^{-1}(\widehat F_i(\Delta_{i,t}))$ & Verifier \\
\bottomrule
\end{tabular}
}
\end{table}

\subsection{Stage 1: Quantile-Normalized Sample Baseline}
\label{sec:s1-quantile}

Task-dependent verifier statistics are rarely well-conditioned: binary outcomes create highly skewed two-point laws, while richer scalar verifiers can be heavy-tailed. Linear standardization carries these differences into the advantage because it is an affine, shape-preserving map. To enforce invariant I1 we therefore map the population reward law $P_{\mathcal{R}_\tau}$ to a fixed reference $P_{\mathcal{N}}=\mathcal{N}(0,1)$ at minimum quadratic cost. When a group has no reward variation, we set $A_i=0$, matching the usual GRPO convention for degenerate groups. Otherwise, Brenier's theorem gives a unique increasing population transport in one dimension; its finite-group plug-in is the mid-rank composition
\begin{equation}
    A_i \;\triangleq\; \Psi_\tau(R_i) \;=\; \Phi^{-1}\!\bigl(\widehat{F}_\tau(R_i)\bigr),
    \qquad
    \widehat{F}_\tau(R_i) \;=\; \frac{\mathrm{rank}_\tau(R_i)-\tfrac{1}{2}}{N_\tau},
    \label{eq:quantile}
\end{equation}
where $N_\tau$ is the group size for task $\tau$, $\widehat{F}_\tau$ is its empirical CDF over the current group, and $\Phi^{-1}$ is the standard normal quantile function. Identical raw rewards (a frequent case under binary verifiers) are mapped to the average of their quantile images so that identical outcomes receive identical credit.

\paragraph{Properties relevant to (S1).}
Three immediate consequences of Eq.~\eqref{eq:quantile} jointly satisfy invariant I1:
\begin{itemize}[nosep,leftmargin=*]
    \item \textit{Inter-task gradient equity.} Under the continuous reward law assumed in Proposition~\ref{prop:equity}, $A_i$ has zero mean and unit variance regardless of the scale or shape of $P_{\mathcal{R}_\tau}$. The empirical rank map likewise removes raw reward magnitudes and higher-order shape from the finite-sample score.
    \item \textit{Outlier and bimodality robustness.} A single anomalously large reward is mapped to $\Phi^{-1}\!\bigl((\mathrm{rank}-0.5)/N\bigr)$, which is bounded by the finite-sample rank grid; the two outcome categories of a binary verifier are mapped to bounded Gaussian quantiles.
    \item \textit{Order preservation.} $\Phi^{-1}\circ\widehat{F}_\tau$ is nondecreasing: distinct rewards receive ordered scores, while tied rewards receive the same score. In particular, $\mathrm{sign}(A_i)$ reflects the verifier's relative verdict within the group and is therefore safe to use as the direction anchor exploited by stage~(S2).
\end{itemize}
The implementation reduces to one $O(N\log N)$ sort per group; no additional hyperparameter is introduced at this stage.

\subsection{Stage 2: Rank-Calibrated, Mass-Preserving Token Credit}
\label{sec:s2-rlsd}

Stage~(S2) must allocate the well-conditioned scalar $A_i$ to the tokens of $y^{(i)}$ without violating invariant I2. Our starting point is the evidence-ratio credit of \citet{yang2026selfdistilledrlvr}: with $P_T \triangleq \pi_\theta(\cdot \mid x, r, y_{<t})$ and $P_S \triangleq \pi_\theta(\cdot \mid x, y_{<t})$, define the stop-gradient log-evidence $\Delta_{i,t} = \mathtt{sg}(\log P_T(y_t) - \log P_S(y_t))$. We then introduce two modifications that make the token channel distribution-free and budget-neutral, mirroring at the token level what stage~(S1) does at the sample level.

\paragraph{Modification 1: within-trajectory rank calibration.}
The raw scale of $\Delta_{i,t}$ is an accident of model calibration, privileged-context strength, and position; trajectories with a confident teacher therefore receive systematically larger token weights. We remove this confound by replacing $\Delta_{i,t}$ with its rank inside the trajectory,
\begin{equation}
    Q_{i,t} \;=\; \Phi^{-1}\!\left(\frac{\mathrm{rank}_{i}(\Delta_{i,t})-\tfrac{1}{2}}{|y^{(i)}|}\right),
    \label{eq:token-rank}
\end{equation}
with ties receiving the average of their quantile images. Because Eq.~\eqref{eq:token-rank} uses only the \emph{ordering} of the evidence within a trajectory, it is invariant to strictly increasing recalibration of $\Delta$ and gives every trajectory the same calibrated evidence scale, which is the token-level analogue of Eq.~\eqref{eq:quantile}.

\paragraph{Modification 2: credit-mass conservation.}
RLSD-style weights redistribute credit but need not preserve its total: if the teacher is confident along the whole trajectory, $\frac{1}{|y|}\sum_t w_{i,t}$ can drift away from $1$, silently amplifying or attenuating the trajectory-level signal that the verifier issued. We decouple \emph{reallocation} from \emph{amplification} by renormalizing the clipped, sign-anchored weight to have mean exactly one over the trajectory:
\begin{equation}
    u_{i,t} = \mathrm{clip}\!\Bigl(e^{\,\mathrm{sign}(A_i)\,Q_{i,t}},\,1-\epsilon_w,\,1+\epsilon_w\Bigr),
    \qquad
    w_{i,t} \;=\; \frac{u_{i,t}}{\frac{1}{|y^{(i)}|}\sum_{s=1}^{|y^{(i)}|} u_{i,s}},
    \label{eq:clipped-credit}
\end{equation}
together with the same short warm-up $w_{i,t}\!\leftarrow\!(1\!-\!\lambda)+\lambda\,w_{i,t}$, $\lambda$ ramping from $0$ to $1$ over the first 50 steps. The clipping idea is inherited from \citet{yang2026selfdistilledrlvr}; the rank calibration of Eq.~\eqref{eq:token-rank} and the mass-preserving normalization of Eq.~\eqref{eq:clipped-credit} are ours. Theorem~\ref{thm:mass} shows the resulting estimator satisfies positivity, a deterministic two-sided bound, exact credit-mass conservation, and sign anchoring.

\subsection{Theory: Equity, Tails, and Conservation}
\label{sec:theory}

The results below formalize the guarantees of the estimator we contribute. Proposition~\ref{prop:equity} applies the classical probability integral transform and inverse-normal quantile map to the GRPO baseline and is a population statement for continuous reward laws; Eq.~\eqref{eq:quantile} is its mid-rank plug-in. The deterministic tail cap of Proposition~\ref{prop:tail} and the conservation, boundedness, and sign-anchoring guarantees of Theorem~\ref{thm:mass} are specific to the composed TACE estimator. The information-asymmetry analysis that motivates direction anchoring is due to \citet{yang2026selfdistilledrlvr} and is cited, not restated.

\begin{proposition}[Exact distributional equity]
\label{prop:equity}
Let task $\tau$ have a continuous reward distribution $P_{\mathcal R_\tau}$ with CDF $F_\tau$, and let $A=\Phi^{-1}(F_\tau(R))$ with $R\sim P_{\mathcal R_\tau}$. Then $A\sim\mathcal N(0,1)$ exactly; in particular $\mathbb E|A|^{p}=m_p$ for every $p\ge 1$, with $m_p$ independent of $\tau$. Affine standardization $(R-\mu)/\sigma$ matches only the first two moments: for every $p>2$ there exist tasks $\tau,\tau'$ with the same mean and variance but $\mathbb E|A_\tau|^{p}\neq\mathbb E|A_{\tau'}|^{p}$.
\end{proposition}

\begin{proposition}[Deterministic outlier bound]
\label{prop:tail}
For a group of $N$ rewards, the estimator of Eq.~\eqref{eq:quantile} satisfies $\max_i |A_i|\le \Phi^{-1}\!\bigl(1-\tfrac{1}{2N}\bigr)$, deterministically and irrespective of the reward scale. Under affine standardization the same quantity can reach $\sqrt{N-1}$ (one outlier against $N{-}1$ tied rewards). For $N{=}8$ the two bounds are $1.53$ vs.\ $2.65$; for $N{=}256$, $2.89$ vs.\ $15.97$.
\end{proposition}

\begin{theorem}[Mass conservation, boundedness, and sign anchoring]
\label{thm:mass}
For every trajectory and every realization of the privileged context $r$, the token weights of Eq.~\eqref{eq:clipped-credit} satisfy
\begin{equation*}
    w_{i,t}>0,\qquad
    \frac{1}{|y^{(i)}|}\sum_{t} w_{i,t}=1,\qquad
    \frac{1-\epsilon_w}{1+\epsilon_w}\le w_{i,t}\le\frac{1+\epsilon_w}{1-\epsilon_w},
\end{equation*}
and consequently $\frac{1}{|y^{(i)}|}\sum_t \hat A_{i,t}=A_i$ and $\mathrm{sign}(\hat A_{i,t})=\mathrm{sign}(A_i)$ for all $t$. Moreover, $Q_{i,t}$ in Eq.~\eqref{eq:token-rank} is invariant to every strictly increasing transform of $\Delta_{i,t}$. In contrast, the unnormalized RLSD weight $(P_T/P_S)^{\mathrm{sign}(A_i)}$ has trajectory mean generically different from $1$ and admits no scale-free description.
\end{theorem}

Propositions~\ref{prop:equity} and~\ref{prop:tail}, together with Theorem~\ref{thm:mass}, separate the two concerns that GRPO conflates: stage~(S1) equalizes the \emph{distribution} of trajectory-level credit across tasks (not just its first two moments) with a hard tail cap, while stage~(S2) redistributes credit across tokens with its total mass and sign fixed. The two normalizations act on disjoint indices, with $i$ across the group and $t$ within the trajectory, so they compose without interference; Appendix~\ref{app:composition} makes this precise.

\subsection{Composed Estimator}
\label{sec:final-objective}

Substituting Eq.~\eqref{eq:quantile} and Eq.~\eqref{eq:clipped-credit} into Eq.~\eqref{eq:two-stage} gives the closed-form TACE advantage:
\begin{equation}
    \hat A_{i,t} \;=\; \Phi^{-1}\!\bigl(\widehat{F}_\tau(R_i)\bigr) \cdot \underbrace{\frac{u_{i,t}}{\frac{1}{|y^{(i)}|}\sum_{s} u_{i,s}}}_{\text{mean-one } w_{i,t}},\qquad
    u_{i,t} = \mathrm{clip}\!\Bigl(e^{\,\mathrm{sign}(A_i)\,Q_{i,t}},\,1-\epsilon_w,\,1+\epsilon_w\Bigr),
    \label{eq:final-advantage}
\end{equation}
which plugs into the standard clipped-surrogate objective:
\begin{equation}
    \mathcal{L}_{\mathrm{TACE}}(\theta) \;=\; \mathbb{E}\!\left[\frac{1}{G}\sum_{i=1}^{G}\frac{1}{|y^{(i)}|}\sum_{t=1}^{|y^{(i)}|} \min\!\bigl(\rho_{i,t}\,\hat A_{i,t},\;\mathrm{clip}(\rho_{i,t},1-\epsilon_{\mathrm{low}},1+\epsilon_{\mathrm{high}})\,\hat A_{i,t}\bigr)\right].
    \label{eq:final-objective}
\end{equation}
Writing $\hat A_{i,t} = f(R_i;\tau)\cdot w_{i,t}$ with $f=\Phi^{-1}\!\circ\widehat F_\tau$, the two stages compose without interference: $f$ is monotone (so $\mathrm{sign}(f)$ is the only feature $w$ reads), and $w_{i,t}>0$ with $\frac{1}{|y|}\sum_t w_{i,t}=1$ (Theorem~\ref{thm:mass}), so stage~(S2) reallocates credit without changing its mass or sign. Beyond GRPO, TACE adds one $O(N\log N)$ sort per group, one rank sort per trajectory, and a single additional privileged-context forward pass.

\section{Empirical Evaluation}
\label{sec:exp}

\subsection{Setup and Baselines}

\noindent \textbf{Training corpus and evaluation suite.}
We train on MMFineReason-123K~\citep{lin2026mmfinereasonclosingmultimodalreasoning}, derived from MMFineReason-1.8M by retaining items that Qwen3-VL-4B-Thinking misses in four independent rollouts. Evaluation covers five multimodal reasoning suites: \textbf{MMMU}~\citep{mmmu} (multi-discipline college-level VQA), \textbf{MathVista}~\citep{mathvista} (math reasoning over visual contexts), \textbf{MathVision}~\citep{mathvision} (competition-level visual math), \textbf{ZeroBench}~\citep{zerobench} (frontier-unsolvable stress test), and \textbf{WeMath}~\citep{wemath} (fine-grained math). All benchmarks report accuracy (\%).
For comparability, we follow the RLSD multimodal protocol for the corpus, backbone, benchmark suite, context length, rollout settings, and quoted baseline rows; the TACE, component-ablation, and late-training sensitivity runs are ours.

\noindent \textbf{Backbone and compared methods.}
The backbone in every run is \textbf{Qwen3-VL-8B-Instruct}~\citep{qwen3vl}. We compare: \textbf{Base VLM} (no additional post-training); \textbf{GRPO}~\citep{grpo}; \textbf{OPSD}~\citep{zhao2026opsd}; \textbf{SDPO}~\citep{hbotter2026reinforcementlearningselfdistillation_sdpo}; \textbf{GRPO+OPSD}, the additive-fusion baseline of \citet{yang2026selfdistilledrlvr}; and \textbf{RLSD}~\citep{yang2026selfdistilledrlvr}. The quoted baseline rows follow \citet{yang2026selfdistilledrlvr}. The bottom \textbf{TACE} row enables both stages, with stage~(S2) implemented as our rank-calibrated, mass-preserving credit rule.

\noindent \textbf{Optimization and hyperparameters.}
Our training pipeline is implemented on top of VERL~\citep{verl} and EasyR1~\citep{easyr1}. Each rollout uses an $8192$-token context window, allocated as a $4096$-token prompt and a $4096$-token response. The group size is $8$ rollouts per prompt, sampled at temperature $1.0$, and the effective optimizer batch contains $256$ prompts. The PPO-style importance-ratio clip is asymmetric, with $(\epsilon_{\text{low}},\epsilon_{\text{high}})=(0.2,0.28)$. For the composed TACE estimator, the stage-S2 hyperparameters follow RLSD: the warm-up coefficient $\lambda$ ramps linearly from $0$ to $1$ over the first $50$ optimization steps, and the credit-channel clip is $\epsilon_w=0.2$. The within-trajectory rank calibration and mean-one normalization introduce no additional hyperparameters. The privileged context $r$ used in stage S2 is the final ground-truth answer, as in RLSD.

\subsection{Experimental Results}

\subsubsection{Five-Benchmark Reasoning Results}
\label{sec:main-results}

\begin{table*}[t]
    \caption{%
        Accuracy after post-training Qwen3-VL-8B-Instruct.
        \textbf{Bold} marks the top score in each column.
        The Base, GRPO, OPSD, SDPO, GRPO+OPSD, and RLSD rows quote the numbers reported by RLSD; the \textit{TACE} row is ours.
    }
    \label{tab:multimodel}
    \centering
    \begin{tabular}{lccccc}
        \toprule
        Method & MMMU & MathVista & MathVision & ZeroBench & WeMath \\
        \midrule
        Base VLM  & 62.44 & 73.80 & 47.37 & 19.76          & 54.10 \\
        GRPO      & 65.11 & 76.20 & 48.82 & 22.60          & 56.57 \\
        OPSD      & 63.82 & 75.10 & 47.53 & 21.06          & 54.95 \\
        SDPO      & 65.11 & 74.00 & 47.27 & \textbf{25.15} & 52.19 \\
        GRPO+OPSD & 63.22 & 75.90 & 48.52 & 22.16          & 54.76 \\
        RLSD      & 67.22 & 78.10 & 52.73 & 24.85          & \textbf{58.00} \\
        \textbf{TACE} \textit{(Ours)} & \textbf{68.67} & \textbf{80.60} & \textbf{56.58} & 18.86 & 45.62 \\
        \bottomrule
    \end{tabular}
\end{table*}

Table~\ref{tab:multimodel} reports the five-benchmark comparison. The composed TACE estimator obtains the strongest results on MMMU, MathVista, and MathVision, while ZeroBench and WeMath remain stronger under existing baselines. This mixed pattern is consistent with the theory: rank-based normalization removes reward-scale and tail effects, but the composed credit rule can still underperform when verifier outcomes are sparse or when the privileged signal provides limited token-level evidence.

\subsubsection{Component Ablation}
\label{sec:component-ablation}

Table~\ref{tab:component-ablation} evaluates two reduced estimators on the first three benchmarks. \emph{w/o S1} replaces the quantile-normalized sample baseline with conventional affine group standardization while retaining the token-level credit channel. \emph{w/o S2} removes the rank-calibrated, mass-preserving token credit and uses uniform token weights while retaining the quantile baseline. Both reduced estimators degrade relative to full TACE, with the largest drop on MathVision in each case (6.05 points without S1 and 3.82 points without S2). Removing S1 also hurts more than removing S2 on every benchmark, indicating that distributional reward calibration is the stronger contributor, while the two stages remain complementary: removing either component breaks the invariant that the composed estimator is designed to preserve.
Appendix~\ref{app:hyperparameter-sensitivity} additionally reports a 14-run, one-factor-at-a-time short-horizon late-training sensitivity study; because each arm uses only 40 two-GPU continuation steps, we treat it as a stability diagnostic rather than a model-selection procedure.

\subsubsection{Qualitative Case Study}
\label{sec:case-study}

We illustrate the trained model's behavior on a WeMath example.

\begin{wrapfigure}{r}{0.34\linewidth}
    \centering
    \includegraphics[width=\linewidth]{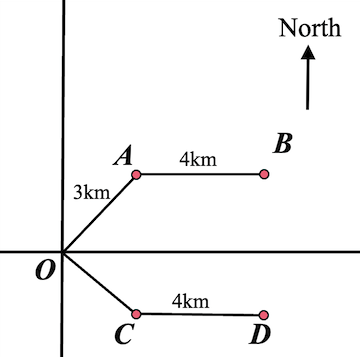}
    \caption{A WeMath example solved correctly by the trained model.}
    \label{fig:wemath-case}
\end{wrapfigure}

\noindent\textbf{Task.}
Figure~\ref{fig:wemath-case} shows a WeMath prompt that asks the model to read two segment lengths from the diagram and report
\[
OA=(\;\cdot\;)\ \text{m}, \qquad AB=(\;\cdot\;)\ \text{m}.
\]
The choices are \(A=(300,400)\), \(B=(400,300)\), \(C=(3000,4000)\), \(D=(4000,3000)\), and \(E=\) ``No correct answer.''

\noindent\textbf{Model reasoning.}
The model first extracts the two image-grounded relations \(OA=3\,\mathrm{km}\) and \(AB=4\,\mathrm{km}\). It then notices the mismatch between the unit in the diagram and the unit requested by the question, applies \(1\,\mathrm{km}=1000\,\mathrm{m}\), and obtains
\[
OA=3\times1000=3000\,\mathrm{m},\qquad
AB=4\times1000=4000\,\mathrm{m}.
\]
These values preserve the original \(OA\)-to-\(AB\) correspondence and select option~C.

\noindent\textbf{Analysis.}
This example tests substantive multimodal reasoning rather than merely matching an answer format. The model must recognize the two labeled segments, bind each numeric value to the correct segment rather than swapping them, detect that the question changes kilometers to meters, and apply the scale factor consistently to both quantities. The distractors intentionally probe the two most likely errors: a tenfold unit-conversion mistake and an order reversal. A correct answer therefore requires coordinated visual extraction, symbol grounding, and numeric transformation.

\begin{table}[t]
    \caption{Component ablation on the first three multimodal reasoning benchmarks. \emph{w/o S1} removes the quantile-normalized sample baseline; \emph{w/o S2} removes the rank-calibrated, mass-preserving token credit.}
    \label{tab:component-ablation}
    \centering
    \small
    \begin{tabular}{lccc}
        \toprule
        Method & MMMU & MathVista & MathVision \\
        \midrule
        TACE w/o S1 & 64.17 & 76.20 & 50.53 \\
        TACE w/o S2 & 65.31 & 77.58 & 52.76 \\
        \textbf{TACE} & \textbf{68.67} & \textbf{80.60} & \textbf{56.58} \\
        \bottomrule
    \end{tabular}
\end{table}

\section{Conclusion and Limitations}
\label{sec:conclusion}
GRPO's sequence-level estimate preserves the shape of each task's reward law, and uniform token credit spreads the resulting signal indiscriminately. \textbf{TACE} addresses these two issues with a sample-level optimal-transport baseline and a trajectory-level credit reallocation rule. Its additional computation consists of sorting and one privileged-context pass. Empirically, improvements on MMMU, MathVista, and MathVision coexist with regressions on ZeroBench and WeMath, indicating that distributional calibration does not by itself resolve sparse outcomes or weak token-level evidence. Our analysis is confined to verifier-only RLVR with a scalar outcome reward on multimodal reasoning; multi-marginal optimal transport for vector-valued rewards is the natural extension.

\bibliographystyle{plainnat}
\bibliography{references}

\newpage
\appendix
\section{Deferred Proofs}
\label{app:deferred-proofs}

This appendix proves the three results of Section~\ref{sec:theory}, states the composition property referenced there, and adds four analyses relevant to finite-sample RLVR: the binary-verifier case, uniform finite-sample accuracy of the empirical rank map, optimality of the transport map, and a deterministic bound on the composed token-level signal. Throughout, $\Phi$ denotes the standard normal CDF, $\Phi^{-1}$ its quantile function, and $z_\alpha=\Phi^{-1}(\alpha)$.

\subsection{Proof of Proposition~\ref{prop:equity} (Distributional Equity)}
\label{app:proof-equity}

Let $R\sim P_{\mathcal R_\tau}$ with continuous CDF $F_\tau$ and set $U=F_\tau(R)$. By the probability-integral transform, $U\sim\mathrm{Unif}[0,1]$. Since $\Phi$ is strictly increasing, $A=\Phi^{-1}(U)$ has CDF
\begin{equation*}
\Pr\{A\le a\} \;=\; \Pr\{\Phi^{-1}(U)\le a\} \;=\; \Pr\{U\le \Phi(a)\} \;=\; \Phi(a),
\end{equation*}
so $A\sim\mathcal N(0,1)$ regardless of $P_{\mathcal R_\tau}$; every absolute moment $\mathbb E|A|^p$ of a standard normal is therefore task-independent.

For the affine part, fix $p>2$ and consider two mean-zero, unit-variance reward laws: a maximally skewed two-point law placing mass $\eta\neq\tfrac12$ on the upper atom, and a symmetric two-point law. Their $p$-th absolute moments differ, and affine standardization preserves this shape difference while matching only the first two moments. \qed

\subsection{Proof of Proposition~\ref{prop:tail} (Outlier Bound)}
\label{app:proof-tail}

With the mid-rank convention of Eq.~\eqref{eq:quantile}, every empirical quantile lies in $[\tfrac{1}{2N},\,1-\tfrac{1}{2N}]$, so $|A_i|\le\Phi^{-1}(1-\tfrac{1}{2N})$ deterministically, independently of the reward scale because $\widehat F_\tau$ depends only on ranks.

For affine standardization, let one reward equal $c$ and the remaining $N-1$ equal $0$. Then $\mu=c/N$ and $\sigma^2=c^2(N-1)/N^2$, so the standardized outlier is $z=\sqrt{N-1}$ for every $c>0$. Thus the affine statistic reaches $\sqrt{N-1}$ exactly, while TACE caps the same configuration at $\Phi^{-1}(1-\tfrac{1}{2N})$. \qed

\subsection{Proof of Theorem~\ref{thm:mass} (Conservation, Boundedness, Sign)}
\label{app:proof-mass}

\emph{(i) Positivity.} $u_{i,t}\ge 1-\epsilon_w>0$ by the clip with $\epsilon_w\in(0,1)$, and the normalizer $\bar u_i=\frac{1}{|y|}\sum_s u_{i,s}$ is a mean of positive numbers, hence $w_{i,t}=u_{i,t}/\bar u_i>0$.

\emph{(ii) Mass conservation.} By construction $\frac{1}{|y|}\sum_t w_{i,t}=\bar u_i/\bar u_i=1$, hence $\frac{1}{|y|}\sum_t\hat A_{i,t}=A_i$.

\emph{(iii) Boundedness.} Since each $u_{i,s}\in[1-\epsilon_w,1+\epsilon_w]$, the mean satisfies $\bar u_i\in[1-\epsilon_w,1+\epsilon_w]$, whence $(1-\epsilon_w)/(1+\epsilon_w)\le w_{i,t}\le(1+\epsilon_w)/(1-\epsilon_w)$, a bound depending only on $\epsilon_w$.

\emph{(iv) Sign anchoring.} $w_{i,t}>0$ gives $\mathrm{sign}(\hat A_{i,t})=\mathrm{sign}(A_i)$ for every realization of $r$. The use of positivity to keep the verifier in control of the update direction originates in the leakage-free analysis of \citet{yang2026selfdistilledrlvr}; the present argument covers our rank-calibrated, mean-normalized weight, for which positivity is preserved under clipping and normalization.

\emph{(v) Scale invariance.} If $g$ is strictly increasing, ranks are unchanged, so $Q_{i,t}$, $u_{i,t}$, $w_{i,t}$, and $\hat A_{i,t}$ are unchanged. The unnormalized RLSD weight $(P_T/P_S)^{\mathrm{sign}(A_i)}=\exp(\mathrm{sign}(A_i)\Delta_{i,t})$ changes with the calibration of $\Delta$, and its trajectory mean need not equal $1$. \qed

\subsection{Composition Property}
\label{app:composition}

The stage-(S1) map acts on the sample index $i$ and depends only on $\{R_j\}$; the stage-(S2) map acts on the token index $t$ and reads stage~(S1) only through the fixed scalar $A_i$ and $\mathrm{sign}(A_i)$. Fixing $\{R_j\}$ therefore activates Theorem~\ref{thm:mass} verbatim, while Propositions~\ref{prop:equity} and~\ref{prop:tail} continue to hold for the trajectory-level scalar. Neither map reads a quantity the other is free to change, so cross-task equity and within-trajectory conservation hold simultaneously. \qed

\section{Additional Theoretical Guarantees}
\label{app:additional-theory}

\subsection{Binary and Tied Rewards}
\label{app:binary}

Binary verifiers are the leading RLVR case not covered by the continuity assumption in Proposition~\ref{prop:equity}. The mid-rank rule nevertheless has an exact closed form.

\begin{proposition}[Binary rewards]
\label{prop:binary}
Suppose a group contains $N$ rewards, $K$ equal to $1$ and $N-K$ equal to $0$, with $1\le K\le N-1$. Under the tied mid-rank convention,
\begin{equation}
    A_0=\Phi^{-1}\!\left(\frac{N-K}{2N}\right),
    \qquad
    A_1=\Phi^{-1}\!\left(1-\frac{K}{2N}\right),
    \label{eq:binary-scores}
\end{equation}
independently of the numeric labels assigned to success and failure. In particular,
\begin{equation*}
    \max\{|A_0|,|A_1|\}\le z_{1-\frac{1}{2N}}.
\end{equation*}
The corresponding affine scores are $\sqrt{(N-K)/K}$ and $-\sqrt{K/(N-K)}$; their magnitudes grow without bound as $K\to 0$ or $K\to N$, whereas Eq.~\eqref{eq:binary-scores} remains bounded for every nondegenerate $K$.
\end{proposition}

\begin{proof}
The failed samples occupy ranks $1,\ldots,N-K$, whose average mid-rank is $(N-K+1)/2$, giving empirical mass $(N-K)/(2N)$. The successful samples occupy ranks $N-K+1,\ldots,N$, whose average mid-rank is $(2N-K+1)/2$, giving empirical mass $1-K/(2N)$. Both masses lie in $[1/(2N),1-1/(2N)]$, which gives the stated cap. The affine formulas follow by direct standardization of the two sample values. When $K=N-1$, for example, the successful affine score is exactly $\sqrt{N-1}$, while the successful rank score is $z_{1-1/(2N)}$.
\end{proof}

Equation~\eqref{eq:binary-scores} also yields exact conditional moments. Given $K$,
\begin{equation}
    \bar A_K=\frac{K A_1+(N-K)A_0}{N},\qquad
    V_K=\frac{K(A_1-\bar A_K)^2+(N-K)(A_0-\bar A_K)^2}{N},
    \label{eq:binary-moments}
\end{equation}
and both are determined by $(N,K)$ alone. Thus two binary tasks with the same group size and success count receive the same conditional score geometry even if their raw reward semantics differ. The rank map is intentionally nonlinear, so $\bar A_K$ need not equal zero when $K\ne N/2$; it trades exact centering for a deterministic, scale-free tail cap. Affine standardization achieves exact centering but pays for it with the $1/\sqrt{K}$ and $1/\sqrt{N-K}$ singularities at the boundary.

\subsection{Finite-Sample Accuracy of the Rank Baseline}
\label{app:finite-sample}

The next result quantifies the plug-in error in Eq.~\eqref{eq:quantile}. It is a nonasymptotic statement; no large-group approximation is used.

\begin{proposition}[High-probability quantile accuracy]
\label{prop:finite-sample}
Let $R_1,\ldots,R_N$ be i.i.d.\ draws from a continuous CDF $F$, and let $\widehat F$ denote their mid-rank empirical CDF. Fix
\begin{equation*}
    \eta\in\left(\frac{1}{2N},\frac14\right)
\end{equation*}
and an index $i$. With probability at least
\begin{equation*}
    1-2\exp\!\left(-2N\left(\eta-\frac{1}{2N}\right)^2\right)-4\eta,
\end{equation*}
the sample lies in the interior $F(R_i)\in[2\eta,1-2\eta]$ and
\begin{equation}
    \left|\Phi^{-1}(\widehat F(R_i))-\Phi^{-1}(F(R_i))\right|
    \le
    \frac{\eta}{\phi\!\left(z_{1-\eta}\right)},
    \label{eq:finite-sample-bound}
\end{equation}
where $\phi$ is the standard normal density. Consequently, for every $p\ge1$, the finite-sample absolute score differs from its population counterpart on the same event by at most
\begin{equation*}
    \left|
        \left|\Phi^{-1}(\widehat F(R_i))\right|^p
        -
        \left|\Phi^{-1}(F(R_i))\right|^p
    \right|
    \le
    \left(\frac{\eta}{\phi(z_{1-\eta})}\right)^p,
\end{equation*}
where the population score is standard normal.
\end{proposition}

\begin{proof}
Let $\bar F(u)=N^{-1}\sum_j\mathbf 1\{R_j\le u\}$ denote the ordinary empirical CDF. The mid-rank version differs from it by at most $1/(2N)$ uniformly:
\begin{equation*}
    \sup_u|\widehat F(u)-\bar F(u)|\le\frac{1}{2N}.
\end{equation*}
By the Dvoretzky--Kiefer--Wolfowitz inequality,
\begin{equation*}
    \Pr\left(\sup_u|\bar F(u)-F(u)|>\eta-\frac{1}{2N}\right)
    \le
    2\exp\!\left(-2N\left(\eta-\frac{1}{2N}\right)^2\right).
\end{equation*}
On this event, $\sup_u|\widehat F(u)-F(u)|\le\eta$. Since $F(R_i)\sim\mathrm{Unif}[0,1]$ under the continuous-law assumption, the probability that it falls outside $[2\eta,1-2\eta]$ is exactly $4\eta$. On the intersection of these two events, $\widehat F(R_i)\in[\eta,1-\eta]$. The inverse normal CDF is differentiable on this interval, and its largest derivative there is $1/\phi(z_{1-\eta})$ because the normal density is minimized at the interval endpoints. The mean value theorem gives Eq.~\eqref{eq:finite-sample-bound}. Finally, $||x|-|y||^p\le|x-y|^p$ for $p\ge1$ gives the displayed score deviation.
\end{proof}

The bound is intentionally conservative in the Gaussian tails, where $\phi(z_{1-\eta})$ is small. Its purpose is not a tight numerical rate but a distribution-free guarantee: the probability failure term depends only on $N$ and $\eta$, not on the reward scale, support, or task.

\subsection{One-Dimensional Transport Optimality}
\label{app:ot}

The population map used in stage~(S1) is not merely a monotone normalization; it is the canonical solution of the quadratic transport problem.

\begin{proposition}[Quadratic optimality]
\label{prop:ot}
Let $\mu$ be a continuous reward law with CDF $F$ and finite second moment, and let $\nu=\mathcal N(0,1)$. Among all measurable maps $T$ such that $T_\#\mu=\nu$, the monotone map
\begin{equation*}
    T^\star(r)=\Phi^{-1}(F(r))
\end{equation*}
minimizes $\mathbb E\{(R-T(R))^2\}$. If $\mu$ is continuous, this monotone map is unique $\mu$-a.s. The minimum equals $W_2^2(\mu,\nu)$.
\end{proposition}

\begin{proof}
The pair $(R,T^\star(R))$ is comonotone because $r\mapsto F(r)$ and $u\mapsto\Phi^{-1}(u)$ are both nondecreasing. For square-integrable variables with the prescribed marginals, the Hardy--Littlewood rearrangement inequality says that a comonotone coupling maximizes $\mathbb E[XY]$, hence minimizes
\begin{equation*}
    \mathbb E[(X-Y)^2]=\mathbb E[X^2]+\mathbb E[Y^2]-2\mathbb E[XY].
\end{equation*}
Thus $(R,T^\star(R))$ is an optimal coupling, so $T^\star$ minimizes the transport cost among maps. In one dimension, an optimal map may be chosen increasing; when $\mu$ has no atoms, this gives uniqueness up to $\mu$-null sets. Finally, continuity makes $F(R)\sim\mathrm{Unif}[0,1]$, so the induced law has CDF
\begin{equation*}
    \Pr\{T^\star(R)\le a\}
    =
    \Pr\{F(R)\le\Phi(a)\}
    =
    \Phi(a),
\end{equation*}
so it indeed has target $\nu$.
\end{proof}

This supplies the precise sense in which TACE is distribution-aware: rather than matching two moments, it selects the monotone transport to a common reference law at minimal quadratic cost. The empirical mid-rank map in Eq.~\eqref{eq:quantile} is its usual plug-in approximation, with the accuracy controlled in Appendix~\ref{app:finite-sample}.

\subsection{Deterministic Bound on the Composed Signal}
\label{app:composed-bound}

Combining the two stages yields a scale-free bound on every token-level coefficient entering the clipped objective.

\begin{proposition}[Bounded composed advantage]
\label{prop:composed-bound}
Let $z_N=z_{1-1/(2N)}$ and $\rho_w=(1+\epsilon_w)/(1-\epsilon_w)$. For every reward realization, trajectory, and token,
\begin{equation}
    |\hat A_{i,t}|
    \le
    z_N\rho_w.
    \label{eq:composed-bound}
\end{equation}
Moreover,
\begin{equation}
    \frac1{|y^{(i)}|}\sum_t \hat A_{i,t}^2
    \le
    z_N^2\rho_w^2.
    \label{eq:composed-second}
\end{equation}
The bound is independent of reward magnitude, reward support, trajectory length, privileged evidence calibration, and task identity.
\end{proposition}

\begin{proof}
Proposition~\ref{prop:tail} gives $|A_i|\le z_N$, while Theorem~\ref{thm:mass}(iii) gives $0<w_{i,t}\le\rho_w$. Since $\hat A_{i,t}=A_iw_{i,t}$, Eq.~\eqref{eq:composed-bound} follows immediately. Averaging the tokenwise bound gives Eq.~\eqref{eq:composed-second}.
\end{proof}

This upgrades the sample-level outlier cap to a uniform bound on the training signal. In contrast, affine standardization can make a single trajectory coefficient as large as $\sqrt{N-1}$, and an unnormalized evidence-ratio channel can further multiply it by a trajectory-dependent factor with no deterministic cap. Equation~\eqref{eq:composed-bound} controls both sources simultaneously; the only remaining quantity in the update is the policy ratio $\rho_{i,t}$, whose range is governed separately by the PPO clip.

\section{Hyperparameter Sensitivity}
\label{app:hyperparameter-sensitivity}

This appendix reports a short-horizon late-training sensitivity study with 14 continuation runs under a shorter, downscaled protocol. The study is used to characterize stability, not to select a final configuration. Three runs repeat the default TACE configuration with different data seeds; the remaining eleven runs change one factor at a time. Each arm branches from the same late-training checkpoint and runs for 40 additional optimization steps with two GPUs under the same trajectory-budget protocol; the checkpoint obtained after these continuation steps is evaluated. Warm-up schedules are measured relative to the start of continuation, and teacher refresh denotes synchronizing the privileged-context branch with the current policy at the stated interval. Thus, the default 50-step and the 100-step warm-up schedules do not finish within this 40-step window. The default uses full TACE, $\epsilon_w=0.2$, a 50-step credit warm-up, teacher refresh every 20 steps, rollout group size $8$, and data seed~1. Group-size arms adjust training and mini-batch sizes to keep the trajectory budget comparable.

We report the mean training reward over the final ten continuation steps and a 262-example WeMath choice-CoT pilot. All entries come from the same late-training sensitivity study and its reduced TACE variants; no external reference model is used. We omit MathVision because its full evaluation is already reported in the main results and is not directly comparable to this shorter two-GPU continuation protocol.

\begin{table}[t]
    \caption{Hyperparameter configuration of each sensitivity arm.}
    \label{tab:hyperparameter-configurations}
    \centering
    \small
    \resizebox{\textwidth}{!}{%
    \begin{tabular}{llrrrrr}
        \toprule
        Arm & Estimator & $\epsilon_w$ & Warm-up & Teacher refresh & Group size & Seed \\
        \midrule
        \texttt{base\_seed1} & TACE & 0.2 & 50 & 20 & 8 & 1 \\
        \texttt{base\_seed2} & TACE & 0.2 & 50 & 20 & 8 & 2 \\
        \texttt{base\_seed3} & TACE & 0.2 & 50 & 20 & 8 & 3 \\
        \texttt{s1\_only} & TACE w/o S2 & -- & -- & -- & 8 & 1 \\
        \texttt{s2\_only} & TACE w/o S1 & 0.2 & 50 & 20 & 8 & 1 \\
        \texttt{eps010} & TACE & 0.1 & 50 & 20 & 8 & 1 \\
        \texttt{eps030} & TACE & 0.3 & 50 & 20 & 8 & 1 \\
        \texttt{warm000} & TACE & 0.2 & 0 & 20 & 8 & 1 \\
        \texttt{warm025} & TACE & 0.2 & 25 & 20 & 8 & 1 \\
        \texttt{warm100} & TACE & 0.2 & 100 & 20 & 8 & 1 \\
        \texttt{refresh010} & TACE & 0.2 & 50 & 10 & 8 & 1 \\
        \texttt{refresh050} & TACE & 0.2 & 50 & 50 & 8 & 1 \\
        \texttt{group4} & TACE & 0.2 & 50 & 20 & 4 & 1 \\
        \texttt{group16} & TACE & 0.2 & 50 & 20 & 16 & 1 \\
        \bottomrule
    \end{tabular}%
    }
\end{table}

\begin{table}[t]
    \caption{Late-training sensitivity results. Reward is the mean over the final ten continuation steps. Clip is the mean response-token clipping ratio, and time is mean wall-clock time per continuation step.}
    \label{tab:hyperparameter-results}
    \centering
    \small
    \resizebox{\textwidth}{!}{%
    \begin{tabular}{lrrrr}
        \toprule
        Arm & Final-10 reward & WeMath pilot & Clip (\%) & Step time (s) \\
        \midrule
        \texttt{base\_seed1} & 0.3957 & 74.43 & 24.03 & 344.7 \\
        \texttt{base\_seed2} & 0.4159 & 78.63 & 25.50 & 361.4 \\
        \texttt{base\_seed3} & 0.3573 & 76.72 & 36.53 & 345.4 \\
        \texttt{s1\_only} & 0.4165 & 74.81 & 21.38 & 298.6 \\
        \texttt{s2\_only} & 0.3841 & 70.61 & 24.92 & 344.7 \\
        \texttt{eps010} & 0.4199 & 74.43 & 21.38 & 371.0 \\
        \texttt{eps030} & 0.3985 & 75.95 & 23.44 & 371.2 \\
        \texttt{warm000} & 0.3942 & 72.14 & 33.61 & 374.2 \\
        \texttt{warm025} & 0.3932 & 73.66 & 25.23 & 368.9 \\
        \texttt{warm100} & 0.4020 & 75.19 & 30.11 & 382.3 \\
        \texttt{refresh010} & 0.4022 & 72.52 & 23.10 & 346.2 \\
        \texttt{refresh050} & 0.3878 & 75.57 & 26.13 & 344.7 \\
        \texttt{group4} & 0.3457 & 76.72 & 25.43 & 375.5 \\
        \texttt{group16} & 0.3513 & 72.90 & 24.23 & 344.3 \\
        \bottomrule
    \end{tabular}%
    }
\end{table}

\begin{table}[t]
    \caption{Variation among the three repeated default-configuration seeds. The standard deviation uses the $n-1$ denominator.}
    \label{tab:baseline-seed-variation}
    \centering
    \small
    \begin{tabular}{lrrl}
        \toprule
        Metric & Mean & Sample SD & Range \\
        \midrule
        Final-10 reward & 0.3896 & 0.0298 & 0.3573--0.4159 \\
        WeMath pilot & 76.59 & 2.09 & 74.43--78.63 \\
        \bottomrule
    \end{tabular}
\end{table}

Seed variation is a substantial source of uncertainty in this short sweep. Repeated controls span $0.0586$ final-ten reward units and $4.20$ WeMath points, so single-run perturbations smaller than these ranges should not be interpreted as reliable improvements. The two largest reward gains are modest relative to control noise: $\epsilon_w=0.1$ improves reward by $1.02$ seed standard deviations and S1-only by $0.90$, but neither improves the WeMath pilot over the repeated-control mean or best seed.

WeMath differences should also be read descriptively. The repeated controls span $74.43$--$78.63$, and the best score belongs to a control seed; no non-control arm exceeds it, although \texttt{group4} ties the third control at $76.72$. This pattern is consistent with short-run variation rather than a causal effect of any one perturbation. Moving group size away from $8$ is negative for training reward in both point estimates: group sizes $4$ and $16$ reduce final-ten reward by $1.48$ and $1.29$ seed standard deviations, respectively. Warm-up and teacher-refresh effects are non-monotone and do not support a clear directional recommendation.

Overall, the sweep shows that short-run reward does not reliably rank external WeMath performance: the highest-reward arms remain near the repeated-control mean on WeMath, while the best WeMath score belongs to a control seed. We therefore use these runs as evidence that TACE's behavior should be assessed across seeds and external benchmarks, not as a basis for changing the main hyperparameters after observing validation outcomes.
\end{document}